\documentclass[review,sort&compress]{elsarticle}

\usepackage{amssymb,amsmath,amsthm}
\newtheorem{proposition}{Proposition}
\usepackage{graphicx}
\usepackage{url}

\usepackage{color}
\usepackage{booktabs}
\usepackage{tabularx}
\usepackage{makecell}
\usepackage{graphicx,subfigure}
\usepackage{soul}

\usepackage{pgfplots}
\pgfplotsset{width=7cm,compat=1.8}

\usepackage{bm}
\usepackage{amsmath}
\usepackage{pgfplots}
\usepackage{wrapfig}

\newcommand{\cmark}{\ding{51}}%
\newcommand{\xmark}{\ding{55}}%

\journal{Medical Image Analysis}

\biboptions{authoryear}

\begin{document}

\begin{frontmatter}

\title{Local Rotation Invariance in 3D CNNs}


\author[address1]{Vincent Andrearczyk$^{*,}$}
\cortext[mycorrespondingauthor]{Corresponding author}
\ead{vincent.andrearczyk@hevs.ch}

\author[address1,address3]{Julien Fageot}

\author[address1,address4]{Valentin Oreiller}
\author[address5]{Xavier Montet}
\author[address1,address4]{Adrien Depeursinge}

\address[address1]{Institute of Information Systems, University of Applied Sciences Western Switzerland (HES-SO), Sierre, Switzerland}
\address[address3]{Harvard School of Engineering and Applied Sciences, Cambridge, Massachusetts}
\address[address4]{Centre Hospitalier Universitaire Vaudois (CHUV), Lausanne, Switzerland}
\address[address5]{Hopitaux Universitaires de Gen\`eve (HUG), Geneva, Switzerland}

\begin{abstract}
Locally Rotation Invariant (LRI) image analysis was shown to be fundamental in many applications and in particular in medical imaging where local structures of tissues occur at arbitrary rotations. LRI constituted the cornerstone of several breakthroughs in texture analysis, including Local Binary Patterns (LBP), Maximum Response 8 (MR8) and steerable filterbanks. Whereas globally rotation invariant Convolutional Neural Networks (CNN) were recently proposed, LRI was very little investigated in the context of deep learning.
LRI designs allow learning filters accounting for all orientations, which enables a drastic reduction of trainable parameters and training data when compared to standard 3D CNNs. In this paper, we propose and compare several methods to obtain LRI CNNs with directional sensitivity. Two methods use orientation channels (responses to rotated kernels), either by explicitly rotating the kernels or using steerable filters.
These orientation channels constitute a locally rotation equivariant representation of the data. Local pooling across orientations yields LRI image analysis. Steerable filters are used to achieve a fine and efficient sampling of 3D rotations as well as a reduction of trainable parameters and operations, thanks to a parametric representations involving solid Spherical Harmonics (SH), which are products of SH with associated learned radial profiles.
Finally, we investigate a third strategy to obtain LRI based on rotational invariants calculated from responses to a learned set of solid SHs. The proposed methods are evaluated and compared to standard CNNs on 3D datasets including synthetic textured volumes composed of rotated patterns, and pulmonary nodule classification in CT. The results show the importance of LRI image analysis while resulting in a drastic reduction of trainable parameters, outperforming standard 3D CNNs trained with rotational data augmentation. 

\end{abstract}
\end{frontmatter}

\section{Introduction}
\label{sec:intro}


Convolutional Neural Networks (CNNs) have been successfully used in various studies to analyze textures.
By construction, CNN operations are translation equivariant, thus particularly adapted to image analysis where objects of interest have arbitrary locations.
In this paper, we propose to incorporate Local Rotation Invariance (LRI) into the CNN architecture, which is known to be crucial for texture analysis and biomedical applications in general because objects and patterns of interest have most often arbitrary orientations~(\cite{depeursinge2018biomedical}).

Globally Rotation Invariant (RI) CNNs have recently been studied, making use of group theory to maintain rotation equivariance throughout the layers.
The 2D Group equivariant CNNs (G-CNN\footnote{When referring to G-CNNs, we consider discrete designs with right angle rotations, while it is defined as a more general framework including continuous designs in~\cite{bekkers2019b}.}), developed in \cite{CoW2016b}, uses rotated (right-angles only) versions of the filters together with appropriate channels permutations.
RI is then obtained by pooling across orientation channels after the last convolutional layer.
3D G-CNNs were shown to improve detection of pulmonary nodule detection in \cite{winkels2019pulmonary} and classification of 3D textures in \cite{andrearczyk2018rotational}, yet the latter study motivated the use of a finer rotation sampling than right-angle rotations to capture realistic arbitrary 3D orientations of directional patterns.
G-CNNs achieve equivariance with respect to \emph{finite} subgroups of the rotation group, which constitutes a bottleneck in 3D.
In 2D, an arbitrary sampling of rotations can be used in a group equivariant approach~(\cite{bekkers2018roto}), while the number of 3D finite rotation groups is restrained.
Both 2D harmonic networks~(\cite{WGT2016}) and 2D steerable CNNs~(\cite{weiler2017learning}) present similarities with the method proposed in this paper although in the 2D domain.
Some recent work consider neural networks on non-Euclidian domains~(\cite{kondor2018generalization}), in particular in the $2$-dimensional sphere, where the invariance to rotations plays a crucial role as in~\cite{kondor2018clebsch} and \cite{cohen2018spherical}.
Finally, 3D steerable CNNs
such as proposed in~\cite{weiler20183d} are very general architectures implementing global equivariance to rotations on the network, and the convolutional layer considered in this paper is covered by their design although not specifically investigated. In particular, the proposed LRI layers are specialized instances of a discrete~(\cite{winkels2019pulmonary}) and steerable G-CNN~(\cite{weiler20183d}).
We differ from their work by making an angular max-pooling after the first convolution layer, which exploits the steerability of the filters, and more importantly, focuses on the sought-after local invariances. While G-CNNs can encode complex objects, we focus on textures with local patterns.

In the above-mentioned approaches, global rotation equivariance is maintained all along with the layers (see Fig.~\ref{fig:lri}, left), and invariance is obtained by using orientation pooling at the end of the network after spatial average pooling.
Global RI is fundamental in various applications, e.g. to analyze pictures taken with arbitrary orientations of the camera.
However, most images are composed of well-defined substructures having arbitrary orientations.
For instance, patterns of interest in medical imaging modalities such as Computed Tomography (CT) and Magnetic Resonance Imaging (MRI) consist of tissue alterations with characteristic 3D textures signatures including necrosis, angiogenesis, fibrosis, or cell proliferation~(\cite{GGG2013}).
These alterations induce imaging signatures such as blobs, intersecting surfaces and curves.
These local low-level patterns are characterized by discriminative directional properties and have arbitrary 3D orientations, which requires combining directional sensitivity with LRI.
When compared to equivariant designs, LRI allows to discard the information on local pattern orientation, resulting in more lightweight CNNs.
\begin{figure}[h!]
	\centering
	\includegraphics[width=1\textwidth]{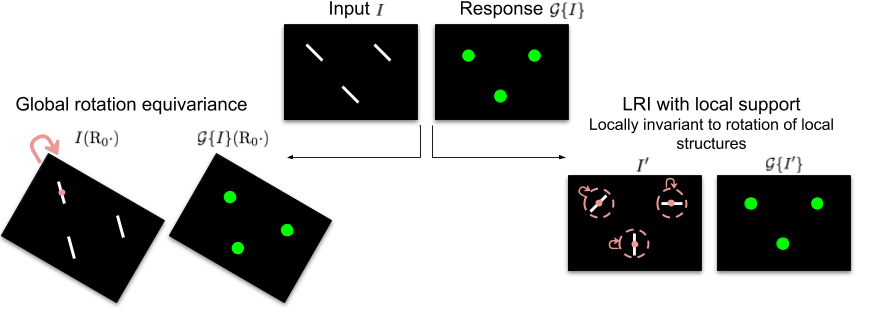}
	\caption{Illustration of global RI and LRI in 2D. Rotating local structures (i.e. three white segments) in the input $I$ results in the input $I'$ on the right. The green dots illustrate the equivariant/invariant responses. 
  Local and global rotations are shown in red and the local support $G$ of the operator $\mathcal{G}$ (see Section~\ref{sec:equiv_local_oper}) is represented as a dashed red line. It is worth noting that our CNN architecture will both present a global equivariance and a local invariance to rotations.
  Best viewed in color.}
	\label{fig:lri}
\end{figure}

However, RI is often antagonistic with the aim of being sensitive to directional features.
For instance, a spatial image operator that is purely convolutional is equivariant to rotations if and only if the filter is isotropic (see Section~\ref{sec:equiv_local_oper} and~\cite{cohen2019general,bekkers2019b}), therefore insensitive to the directional features of the input signal.
It follows that operators combining LRI and directional sensitivity (i.e. non-isotropic) require using more complex designs such as MR8 (\cite{VaZ2005}), local binary patterns (\cite{OPM2002}), steerable Riesz wavelets (\cite{DicenteCid2017}), circular or Spherical Harmonic (SH) invariants (\cite{depeursinge2018rotation}), sparse coding with steerable atoms~\cite{MUD2020} and scattering transform~(\cite{eickenberg2017solid}).
These designs were widely used in hand-crafted texture analysis (\cite{Liu2019,depeursinge2018biomedical}).

In this paper, we propose three 3D CNN architectures that are both globally equivariant and locally invariant to rotations (see Fig.~\ref{fig:lri} for an illustration in  2D), and can combine this with directionally sensitive image analysis.
This can be achieved by convolving with rotated filters (i.e. G-convolution~(\cite{winkels2019pulmonary}, referred to as G-LRI), steered responses to SHs (\cite{andrearczyk2019exploring}, referred to as S-LRI), or Solid Spherical Energy (SSE) invariants calculated from SH responses (\cite{andrearczyk2019solid}, referred to as SSE-LRI). Experiments in Section~\ref{sec:exp} show the benefit of LRI designs over standard CNNs (Tables~\ref{tab:res_all_synthetic}-\ref{tab:ri_lri_nlst}) and globally rotation invariant designs (Tables~\ref{tab:ri_lri_synthetic}, \ref{tab:ri_lri_nlst}) on synthetic textures and lung nodule datasets where local patterns occur at random orientations.
%


\section{Methods}
\label{sec:methods}
This section is organized as follows.
After clarifying mathematical notations in Section~\ref{sec:notations}, we first define a general 3D LRI operator in Section~\ref{sec:equiv_local_oper}.
Section~\ref{sec:steerableFilters_and_SHs} introduces the mathematical tools used in this paper: steerable filters and spherical harmonics.
The different methods that we use to implement the operator, namely G-LRI (based on the G-CNN), S-LRI (Steerable LRI), SSE-LRI (Solid Spherical Energy LRI) are detailed in Section~\ref{sec:lri}.
We then introduce global RI in Section~\ref{sec:ri}, which is further compared against LRI approaches in Section~\ref{sec:localVSglobalRI}. 
Finally, the discretization, datasets, network architectures and weights initialization are presented in Sections~\ref{sec:discretization}, \ref{sec:datasets}, \ref{sec:architecture} and \ref{sec:initialization} respectively.

\subsection{Notations}\label{sec:notations}
We initially introduce the frameworks in the continuous domain, hence 3D images, filters, and response maps are functions defined over the continuum $\mathbb{R}^3$. 
We shall also discuss the practical discretization of the different methods (Section \ref{sec:discretization}).
Spherical coordinates are defined as $(\rho,\theta,\phi)$ with radius $\rho \geq 0$, elevation angle $\theta \in [0,\pi]$, and horizontal plane angle $\phi \in [0,2\pi)$.
The set of 3D rotations is denoted by $SO(3)$.
A 3D rotation transformation matrix $\mathrm{R}$ can be decomposed as three elementary rotations around $z$, $y'$ and $z''$ axes as $\mathrm{R} = \mathrm{R}_{\alpha}\mathrm{R}_{\beta}\mathrm{R}_{\gamma}$, with the orientation $(\alpha,\beta,\gamma)$ parameterized by the (intrinsic) Euler angles $\alpha \in [0,2\pi)$, $\beta \in [0,\pi]$, and $\gamma \in [0,2\pi)$ respectively. 
We will use interchangeably $\mathrm{R}$ as a rotation transformation acting on $\mathbb{R}^3$ and on the two-dimensional sphere $\mathbb{S}^2$. Finally, the function $\bm{x} \mapsto f(\mathrm{R} \bm{x})$ is denoted by $f(\mathrm{R}\cdot)$.
\subsection{Equivariant Image Operators and Invariant Image Features}\label{sec:equiv_local_oper}
%

We introduce the general class of image operators of interest that will be used in the first layer of our neural network and common between G-LRI, S-LRI and SSE-LRI.
An image operator $\mathcal{G}$ associates to an image $I$ another image, denoted by $\mathcal{G}\{I\}$. 
The following invariance properties will be relevant for our analysis:
\begin{itemize}
    \item An operator $\mathcal{G}$ is globally \emph{equivariant to translations and rotations}, if, for any position $\bm{x}_0 \in \mathbb{R}^3$ and rotation $\mathrm{R}_0 \in SO(3)$, 
\begin{align}
    \mathcal{G}\{ I (\cdot - \bm{x}_0) \}  & = \mathcal{G}\{I\} (\cdot - \bm{x}_0) \quad \text{ for any } \bm{x}_0 \in \mathbb{R}^3, \label{eq:transinv} \\
       \mathcal{G}\{ I (\mathrm{R}_0 \cdot) \} & = \mathcal{G}\{I\} (\mathrm{R}_0 \cdot)  \quad \text{ for any } \mathrm{R}_0 \in SO(3). \label{eq:rotinv}
\end{align}

    In particular, if $\mathrm{R}_{\bm{x}_0}$ is a rotation around $\bm{x}_0 \in \mathbb{R}^3$, we have that $\mathcal{G}\{ I( \mathrm{R}_{\bm{x}_0}\cdot) \} = \mathcal{G}\{I\}(\mathrm{R}_{\bm{x}_0} \cdot)$, as illustrated on the left part of Fig.~\ref{fig:lri}.
    \item An operator $\mathcal{G}$ is \emph{local} if there exists $\rho_0 > 0$ such that, for every $\bm{x}$, the quantity $\mathcal{G} \{I\} (\bm{x})$ only depends on local image values  $I(\bm{y})$ for $\lVert \bm{y} - \bm{x}\rVert \leq \rho_0$.
 
 \end{itemize}
The global equivariance to translations and rotations together with the locality result in the sought-after invariance to local rotations (i.e. LRI) in the following sense: the rotation of an object or localized structure of interest in the image $I$ around a position $\bm{x}$ does not affect the value of $\mathcal{G}\{I\} (\bm{x})$, as illustrated on the right part of Fig. \ref{fig:lri}. We illustrate the different notions for the case of linear convolution operators in the next result.

\begin{proposition}
\label{prop:linearRI} 
Let $\mathcal{G}$ be a linear convolution operator of the form 
\begin{equation}
    \mathcal{G} \{I \} = h * I
\end{equation} 
with $h$ the impulse response of the filter.
Then, $\mathcal{G}$ is globally equivariant to rotations if and only if $h$ is isotropic, \emph{i.e.}, $h(\mathrm{R}_0 \cdot ) = h$ for any rotation $\mathrm{R}_0 \in SO(3)$. 
Moreover, $\mathcal{G}$ is local if and only if $h$ is compactly supported. Therefore, $\mathcal{G}$ is LRI if only if $h$ is compactly supported and isotropic.
\end{proposition}

The proof of Proposition \ref{prop:linearRI} is given in \ref{app:Glinear}.
The result is elementary, and can be deduced using general frameworks, such as~\cite[Theorem 1]{bekkers2019b}. It reveals that linear operators can only fulfill the required equivariances using isotropic filters, which are insensitive to the directional information and thus very limited~(\cite{depeursinge2018rotation}).
The operators used in this paper are therefore non-linear.

\subsection{Steerable Filters and Spherical Harmonics}\label{sec:steerableFilters_and_SHs}
This subsection introduces the mathematical toolbox required to characterize the proposed S-LRI and SSE-LRI approaches, which both rely on parametric kernel representations based on solid SHs.
In particular, we consider filters $f$ expanded in terms of the family of SHs $(Y_{n,m})_{n\geq 0, \ m\in \{-n \ldots n \}}$, where $n$ is called the degree and $m$ the order, and which forms an orthonormal basis for square-integrable functions $g(\theta,\phi)$ on the sphere $\mathbb{S}^2$. We consider finitely many degrees, $N \geq 0$ being the maximal one. The number of elements of a SH family of maximum degree $N$ is $\sum_{n=0}^N (2n+1) = (N+1)^2$.
The expression of SHs can be found in~\ref{app:SH}.
 We say that a function $f : \mathbb{R}^3 \rightarrow \mathbb{R}$ is a solid SH\footnote{Note that the terminology is sometimes used when dealing with radial profiles following power law (\cite{eickenberg2017solid}).} if it is a product of a SH with a purely radial function; that is, if it can be written as $f(\rho, \theta, \phi) = h(\rho) Y_{n}^m(\theta,\phi)$.

The S-LRI uses steerable filters (see Section~\ref{sec:s-lri}), which have the advantage to allow for fast and efficient computation of the LRI representation required for max-pooling over orientations channels to further achieve invariance~(\cite{chenouard20123d,fageot2018principled}).
A filter is steerable if any of its rotated versions can be written as a linear combination of finitely many basis filters~(\cite{freeman1991design,UnC2013}).

In this paper, we consider 3D steerable filters $f : \mathbb{R}^3 \rightarrow \mathbb{R}$ of the form 

\begin{equation} \label{eq:formoffilterinterm}
	f(\rho,\theta,\phi) = \sum_{n=0}^{N} h_n(\rho)  \sum_{m=-n}^{n}{ \mathrm{C}_n[m] Y_{n,m} (\theta,\phi)},
\end{equation}
where the $h_{n}(\rho) \in \mathbb{R}$ are degree-dependent radial profiles and the coefficients $\mathrm{C}_n[m] \in \mathbb{C}$ determine the angular structure of $f$.

Many work deal with steerable filters that are polar-separable. This means that 
$f$ can be decomposed as $f(\rho,\theta,\phi)= h(\rho) g(\theta, \phi)$. A steerable filter of the form \eqref{eq:formoffilterinterm}
is polar-separable if and only if it can be written as
 \begin{equation} \label{eq:formoffilterpolarsep}
	f(\rho,\theta,\phi) = h(\rho) \sum_{n=0}^{N}   \sum_{m=-n}^{n}{ \mathrm{C}_n[m] Y_{n,m} (\theta,\phi)},
\end{equation}	 
with $h$ a single radial profile that captures the radial pattern of the filter.
The polar separable case \eqref{eq:formoffilterpolarsep} is a particular case of \eqref{eq:formoffilterinterm}, it corresponds to the situation when $h_n$ does not depend on $n$. In the sequel, we keep track on the index $n$, which covers both cases.

The condition of $f$ being real is translated into the conditions that $h$ or $h_n$  themselves are real and that the SH coefficients satisfy $\mathrm{C}_n[-m] = (-1)^{m}\overline{\mathrm{C}_n[m]}$ (see \ref{sec:proof1}).


For any rotation $\mathrm{R}\in SO(3)$, the rotated version $Y_{n,m}(\mathrm{R}\cdot)$ of a SH can be expressed as a linear combination of all elements in a degree subspace $n$ as
\begin{equation} 
\label{eq:YnmYnm'}
     Y_{n,m}(\mathrm{R}\cdot) = \sum_{m'=-n}^n \mathrm{D}_{\mathrm{R},n}[m,m'] Y_{n,m'},
 \end{equation} 
where the $\mathrm{D}_{\mathrm{R},n} \in  \mathbb{C}^{(2n+1)\times(2n+1)}$ are the Wigner matrices~(\cite{varshalovich1988quantum}). 
Then, the steerable filter $f$ can be rotated efficiently with any $\mathrm{R} \in SO(3)$ to obtain a set of steered coefficients $\mathrm{C}_{\mathrm{R},n}=\mathrm{D}_{\mathrm{R},n} \mathrm{C}_n$ of $f(\mathrm{R}\cdot)$, with $\mathrm{C}_n = (\mathrm{C}_n[m])_{m \in \{-n , \ldots , n\}}$. 
The rotated filter $f(\mathrm{R}\cdot)$ is given by
\begin{align}
    \label{eq:rotatedfilter}
    f(\mathrm{R}\cdot) (\rho,\theta,\phi) 
    &= \sum_{n = 0}^N   h_n(\rho) \sum_{m=-n}^{n} \sum_{m'=-n}^{n} \mathrm{D}_{\mathrm{R},n}[m,m'] \mathrm{C}_n[m'] Y_{n,m}(\theta,\phi).
\end{align}
From \eqref{eq:rotatedfilter}, we see that any rotated version of $f$ can be computed from the coefficients $(\mathrm{C}_n[m])_{0\leq n \leq N,-n\leq m \leq n}$.

In~\cite{andrearczyk2019exploring}, we only considered polar separable filters, in the sense that $f$ can be written as $f(\rho,\theta,\phi) = h(\rho) g(\theta,\phi)$ with $h : \mathbb{R}^+ \rightarrow \mathbb{R}$ and $g : \mathbb{S}^2 \rightarrow \mathbb{R}$, as is the case in \eqref{eq:formoffilterpolarsep}. 

Using a shared radial profile for all SHs results in a reduction of trainable parameters, at the cost of limited SH parametric approximation capability (restricted to polar separable patterns). The extension to non-polar separable filters of the form \eqref{eq:formoffilterinterm} is an important contribution of this paper.

\subsection{Locally Rotation Invariant 3D CNNs} \label{sec:lri}
This section details the three proposed strategies  to achieve 3D LRI image analysis.
An overview of the three methods is depicted in Fig.~\ref{fig:lri_overview},
and a qualitative comparison is presented in Table~\ref{tab:qualitative}.

\begin{figure}[h!]
	\centering
	\includegraphics[width=1\textwidth]{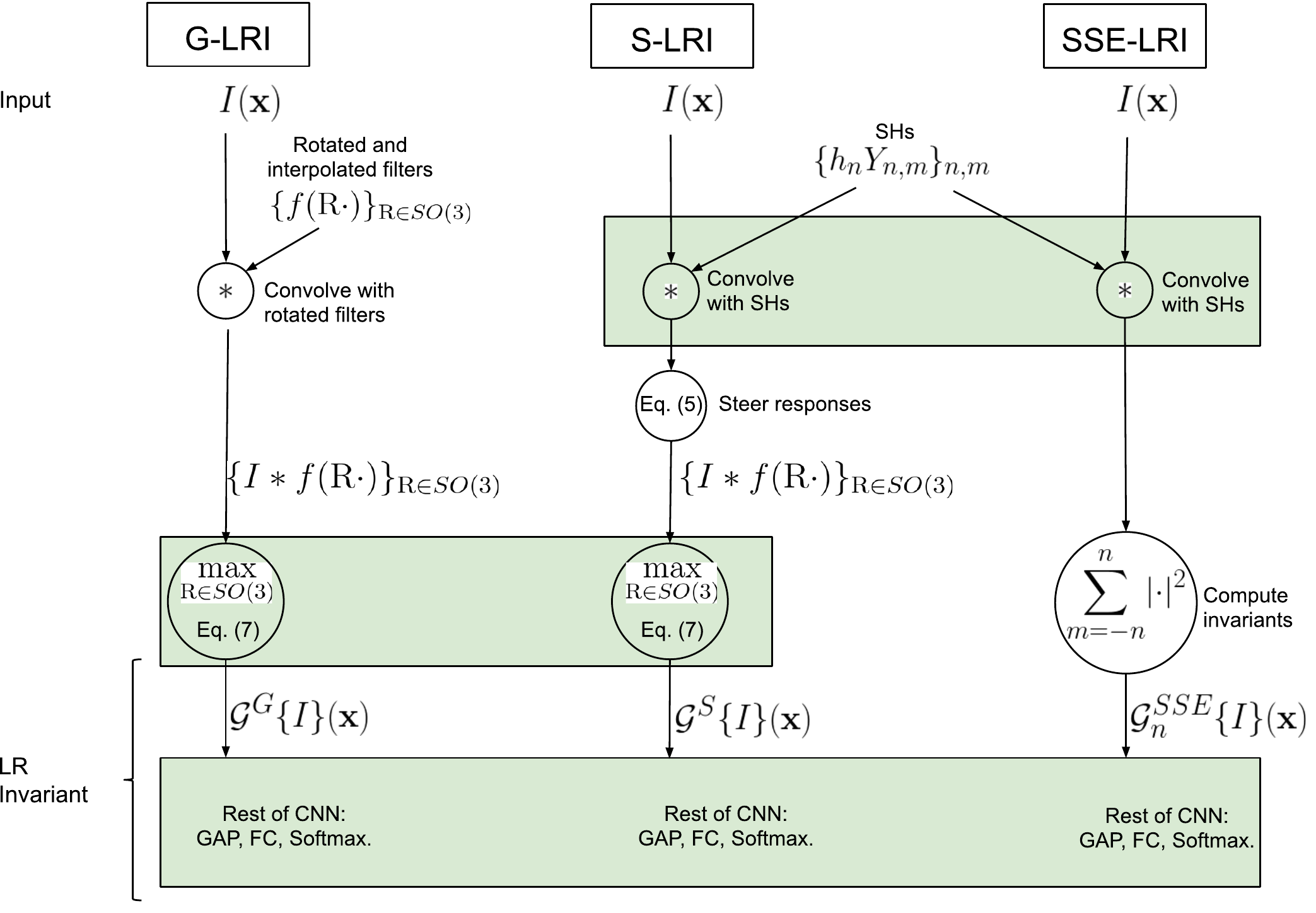}
	\caption{Overview of the methods used to obtain LRI including Group-equivariant LRI (G-LRI), Steerable LRI (S-LRI) and Solid Spherical Energy LRI (SSE-LRI). Operations shared by multiple methods are highlighted in green.}
	\label{fig:lri_overview}
\end{figure}
\begin{table}[h]
\centering
\begin{tabular}{ccccc}
  \bfseries Model & \bfseries LRI & \bfseries \thead{Rotation \\ weight sharing} & \bfseries \thead{Parametric \\ representation} & \bfseries \thead{No need for\\ orientation channels}  \\
  \hline  
  Z3-CNN & \xmark & \xmark & \xmark & \cmark \\
  G-LRI & \cmark & \cmark & \xmark & \xmark \\
  S-LRI & \cmark & \cmark & \cmark & \xmark \\
  SSE-LRI & \cmark & \cmark & \cmark & \cmark \\
  \hline  
  \end{tabular}
\caption{Qualitative comparison of the considered 3D CNN frameworks.}
\label{tab:qualitative}
\end{table}

\subsubsection{G-LRI}\label{sec:g-lri}

The first method to obtain LRI is to use rotated versions of the kernels, i.e. via weight sharing across orientation channels. 
LRI is obtained by max-pooling over the rotations and the corresponding image operator is
\begin{equation}\label{eq:Gg}
    \mathcal{G}^{G} \{I \}(\bm{x}) = \max_{\mathrm{R} \in SO(3)} \left\lvert (I* f(\mathrm{R} \cdot) ) (\bm{x}) \right\rvert,
\end{equation} 
where $f$ is characterized by trainable parameters as in a classical CNN, i.e. full 3D kernels. 
The proof of equivariance to translation and rotation as \eqref{eq:transinv} and \eqref{eq:rotinv} is provided in~\ref{app:equivariance_s}. Moreover, the operator $\mathcal{G}^G$ is local if and only if the filter $f$ has a finite support, what we assume from now.

The idea of max pooling over oriented filter responses has been long used in computer vision, e.g. for template matching with cross-correlation~(\cite{brown1992survey}).
More recently, the idea of rotating the CNN kernels has been widely used in the literature in the context of equivariance to groups of rotations~(\cite{CoW2016b,winkels2019pulmonary,worrall2018cubenet}). In particular, the 3D G-CNN developed in~\cite{winkels2019pulmonary} offers equivariance to groups of 3D rotations. In reference to this work, we refer to this first approach as G-LRI even though we do not propagate the equivariance to deeper layers and neither require to perform operations on finite groups.


\subsubsection{Steerable LRI} \label{sec:s-lri}
S-LRI is a special case of G-LRI for which the computation exploits steerability.
Such S-LRI layers were proposed in~\cite{andrearczyk2019exploring} with polar separable filters only. 
Here we define S-LRI for both polar separable (S-LRI-$h$) and non-polar separable (S-LRI-$h_n$) pattern approximation methods, \eqref{eq:formoffilterpolarsep} and \eqref{eq:formoffilterinterm}, respectively.
As mentioned before, we keep track on the index $n$ which covers both cases.

The S-LRI operator $\mathcal{G}^S \{I\} (\bm{x})$ is obtained by max-pooling over the rotations as in \eqref{eq:Gg}:
\begin{equation}\label{eq:Gs}
    \mathcal{G}^{S} \{I \}(\bm{x}) = \max_{\mathrm{R} \in SO(3)} \left\lvert (I* f(\mathrm{R} \cdot) ) (\bm{x}) \right\rvert,
\end{equation} 
where the filter $f$ is in this case of the form \eqref{eq:formoffilterinterm} and is assumed to have a finite support. The operator $\mathcal{G}^{S}$ is defined identically to $\mathcal{G}^{G}$ in \eqref{eq:Gg} but we use different notations to keep in mind that the parametrization of the filters $f$ differ. As we have seen, the image operators \eqref{eq:Gg} and \eqref{eq:Gs} are  equivariant to rotations and translations. It is moreover local as soon as the $h_n$ have a finite support and, therefore, LRI.

Exploiting \eqref{eq:rotatedfilter}, the convolution $I * (f(\mathrm{R}\cdot))$ is then computed as
\begin{equation}
	I*f(\mathrm{R}\cdot)=\sum_{n = 0}^N \sum_{m=-n}^{n} \left( \sum_{m'=-n}^{n}
  \mathrm{D}_{\mathrm{R},n}[m,m'] \mathrm{C}_n[m'] \right)  \left(I*h_n Y_{n,m}\right) .
\end{equation}
Therefore, one accesses the convolution with any (virtually) rotated version of $f$ by computing $\sum_{n=0}^{N} (2n+1) = (N+1)^2$ convolutions $\left(I*h_n Y_{n,m}\right)$, which we shall exploit for computing the response map of the image operator.
It is worth noting that the case $N=0$ corresponds to filters $f$ that are isotropic, i.e. $f(\mathrm{R}\cdot) = f$ for any $\mathrm{R} \in SO(3)$~(\cite{depeursinge2018rotation}).
As low degrees (e.g. $N=1,2$) are sufficient to construct small filters (see Section \ref{sec:radial_profiles}), the gain becomes substantial over a G-CNN approach for a fine sampling of orientations with a drastic reduction of the number of convolutions. 


In practice, one has a set of steerable filters $f_i$ of the form \eqref{eq:formoffilterinterm} with radial profiles $h_{i,n}$ and coefficients $\mathrm{C}_{i,n}[m]$. 
When compared to the G-LRI, the number of trainable parameters is reduced to $\mathrm{C}_{i,n}[m]$, $h_{i,n}$, and the biases added after orientation pooling (one scalar parameter per output channel $i$).
%

\subsubsection{Solid Spherical Energy LRI} \label{sec:sse-lri}
The use of solid SH representations, i.e. SH representations with radial profiles, provides the opportunity to compute rotational invariants from simple non-linear operations.
Initially proposed in~\cite{andrearczyk2019solid} for polar separable (SSE-LRI-$h$) kernels, we extend the invariants to the non-polar separable (SSE-LRI-$h_n$) case. Here we re-use most of the concepts used for the S-LRI. However, instead of steering, we calculate invariants directly from the responses of the solid SHs, which obviates the need to construct an intermediate (discretized) locally rotation equivariant representation.

After convolution with the image $I$, the responses $I*h_nY_{n,m}$ with $m=-n,\ldots , n$ contain the spectral information of degree $n$, which is used to define the image operator $\mathcal{G}_n^{SSE}$ as
\begin{equation} \label{eq:Gsse}
    \mathcal{G}_n^{SSE} \{I \}(\bm{x}) = \sum_{m=-n}^n \lvert (I * h_n Y_{n,m}) ( \bm{x} ) \rvert^2.
\end{equation}
Let us study the desirable properties of $\mathcal{G}_n^{SSE}$ in the following. 
At a fixed spatial position $\bm{x} \in \mathbb{R}^3$, the projection $(I * h_n Y_{n,m}) ( \bm{x} )= \langle h_n Y_{n,m}, I( \bm{x} - \cdot ) \rangle$ measures the correlation of $h_n Y_{n,m}$ with $I$ at $\bm{x}$.
We call $\mathcal{G}_n^{SSE}\{I\}$ the \emph{SSE response map of degree $n$} of $I$.
The latter are equivariant to translations and (global) rotations as defined in \eqref{eq:transinv} and \eqref{eq:rotinv}. The proof is given in \ref{app:equivariance_sse}.
Note that \eqref{eq:Gsse} defines an operator with rotational equivariance, while being sensitive to directional information via spherical frequencies of degree $n>0$. 
Moreover, the image operator $\mathcal{G}^{SSE}$ is local if and only if the radial profiles $h_n$ have a finite support, what we always assume thereafter.
Finally, more complete invariant quantities can be computed from the solid spherical harmonic representation~(\cite{kazhdan2003rotation,oreiller2020bispectrum,kakarala2012bispectrum}).

\subsection{Global RI}\label{sec:ri}
In this section, we define a global RI layer which will later be used to compare against local invariance.
As defined in Section~\ref{sec:equiv_local_oper} and illustrated in Fig.~\ref{fig:lri} (right), an LRI operator is invariant to local rotations that are not constrained to be identical at every position $\bm{x}_0$.
This is required to characterize important local structures (e.g. textons) having arbitrary and most likely different local orientations.
In this section, we want to compare this LRI with the case where the local rotations use one shared orientation across all positions of the entire image. 
Note that this setting is similar to using standard kernels convolved on the entire image in a regular CNN.
We choose this orientation so that a global RI is achieved (i.e. as illustrated in Fig. \ref{fig:lri} left) but without invariance to local rotations of patterns (Fig. \ref{fig:lri} right).
This global RI can be obtained from equivariant representations with orientation channels by first using a spatial Global Average Pooling (GAP),
followed by max-pooling on the orientation channels. In this
way, the average response is
invariant to global rotations $\mathrm{R}$, resulting in a scalar feature $\mu^{RI}$ given by
\begin{equation}
    \label{eq:gri}
    \mu^{RI} \{I\}= \max_{\mathrm{R} \in SO(3)} \int_{\mathbb{R}^3}\left\lvert (I* f(\mathrm{R} \cdot) ) (\bm{x}) \right\rvert \mathrm{d}\bm{x}.
\end{equation}
Note that the order of the GAP and the orientation max-pooling operations is simply swapped as compared to the aggregation of the G-LRI and S-LRI. We can think of equation \eqref{eq:gri} as finding the rotation of the image $I$ that maximizes the average response to the filter $f$.
Note that this RI layer shares similar ideas with a test-time augmentation. However, the filters are rotated rather than the images, the maximum is taken individually for each filter, and it is also applied at training time.

\subsection{Discretization}\label{sec:discretization}
The discretization of the methods, defined so far in the continuous domain, is necessary for their implementation and naturally introduces an approximation of the invariance properties defined in Section \ref{sec:equiv_local_oper}.

\subsubsection{Rotations Sampling}\label{sec:rotations}
Sampling the rotations, defined so far continuously in the S-LRI and G-LRI approaches, is necessary to compute the invariant responses as in \eqref{eq:Gg} and \eqref{eq:Gs}.  
We sample rotations $\mathrm{R}\in B \subset SO(3)$, where $B$ is a finite subset of sampled rotations. To this end, we uniformly sample orientations as points on the sphere using a triangulation method that iteratively splits octahedron faces to obtain the Euler angles $(\alpha,\beta)$ around $z$ and $y'$ respectively. We then sample the last angle $\gamma$ around $z''$ uniformly between $0$ and $2\pi$. The octahedral group $O$, for instance, is obtained by sampling 6 points on the sphere (i.e. six $(\alpha,\beta)$ pairs) and four values of $\gamma$ to obtain the 24 right-angle rotations. We denote by $M=|B|$ the number of tested rotations. 

In this paper, we evaluate the following sets of rotations: no rotation ($M=1$), Klein's four rotations ($M=4$), octahedral group of rotations ($M=24$) and 72 rotations ($M=72$ with 18 points on the sphere and 4 values of $\gamma$). 
For the G-LRI, we restrict the evaluation to the octahedral group as implemented in \cite{winkels2019pulmonary} as evaluating more rotations becomes computationally too expensive and requires interpolation for non right-angle rotations.
It is worth noting that for both G-LRI and S-LRI designs, LRI being obtained by max-pooling over the $M$ orientation channels after the first convolution, this rotation sampling results in an approximated invariance. Finally, for a discrete G-CNN, it is required that $B=G$ is a finite subgroup of $SO(3)$, which is not needed in our case, since we do not propagate the equivariance to the next layer.

\subsubsection{Naive Filter Discretization}\label{sec:naive_filter}
In the G-LRI (\ref{sec:g-lri}), the filters $f$ are simply voxelized to 3D kernels of $c^3$ voxels as in a standard 3D CNN or G-CNN architecture. All the voxels are trainable parameters that are shared across rotations.

\subsubsection{Radial Profiles}\label{sec:radial_profiles}
In both the S-LRI and SSE-LRI methods, the radial profiles $h_{i,n}$ (and hence the filters $f_i$) have a compact spherical support $G= \{\bm{x}\in \mathbb{R}^3, \lVert \bm{x} \rVert \leq \rho_0\}$, where $\rho_0>0$ is fixed.
For any $i$ and $n$, we consider the voxelized version of the radial profile $h_{i,n}(\rho)$.
The size of the support of the voxelized version is related to the maximum radius $\rho_0$ of the filter in the continuous domain and the level of voxelization.
Due to the isotropic constraint, for a support of $c^3$ voxels, the number of trainable parameters for each $h_{i,n}$ is $\Bigl\lceil \frac{(c-1)}{2} \times\sqrt{3} \Bigr\rceil +1$. 
The values of the filter $f_i(\rho,\theta,\phi)$ over the continuum is deduced from the discretization of the voxelized radial profile on the 3D discrete grid using linear interpolation\footnote{Note that this discretization is not truly isotropic due to the corner effect as the last weights of the radial profile $h_{i,n}$ only affect the corners of the interpolated cubes. While this could be avoided by reducing the length of the radial profile, we favor this implementation for the following reasons.
With right-angle rotations, cubic filters are optimal and do not deteriorate the already approximated rotation invariance.
Besides, the isotropy can easily be learned by forcing the corner weights to zero for finer rotation samplings.}.


%
%
The maximal degree $N$ cannot be taken arbitrarily large once the radial profiles are voxelized. Indeed, the discretized filters $f_i$ are defined over $c^3$ voxels, which imposes the restriction that $N \leq \pi c / 4$, which can be interpreted as the spherical Nyquist frequency. 

%
\begin{figure}[h!]
	\centering
	\includegraphics[width=0.25\textwidth]{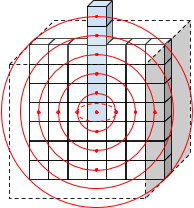}
	\caption{Illustration of a 2D slice of the isotropic radial profiles $h_{i,n}$ with $c=7$. The blue voxels represent the trainable parameters. The rest of the cube is linearly interpolated.}
	\label{fig:rsupport}
\end{figure}

\subsection{Datasets}\label{sec:datasets}
We evaluate the proposed method with two experiments described in the following.

In the first experiment is a sanity check to ensure the relevance of the LRI property. We built a dataset for texture classification containing two classes with 500 synthetic volumes each. The volumes of size $32\times32\times32$ are generated by placing two $7\times7\times7$ patterns, namely a binary segment and a 2D cross with the same norm, at random 3D orientations and random locations with overlap. The number of patterns per volume is randomly set to $\lfloor d(\frac{s_v}{s_p})^3\rfloor$, where $s_v$ and $s_p$ are the sizes of the volume and of the pattern respectively and the density $d$ is drawn from a uniform distribution in the range $[0.1,0.5]$. The two texture classes vary by the proportion of the patterns, i.e. 30\% segments with 70\% crosses for the first class and vice versa for the second class. 800 volumes are used for training and the remaining 200 for testing.
Despite the simplicity of this dataset, some variability is introduced by the overlapping patterns and the linear interpolation of the 3D rotations, making it challenging and more realistic.
A 2D schematic illustration of the 3D synthetic textures is shown in Fig. \ref{fig:synthetic}.
\begin{figure}[h!]
\centering
\subfigure[class 1]{\includegraphics[width=0.15\linewidth]{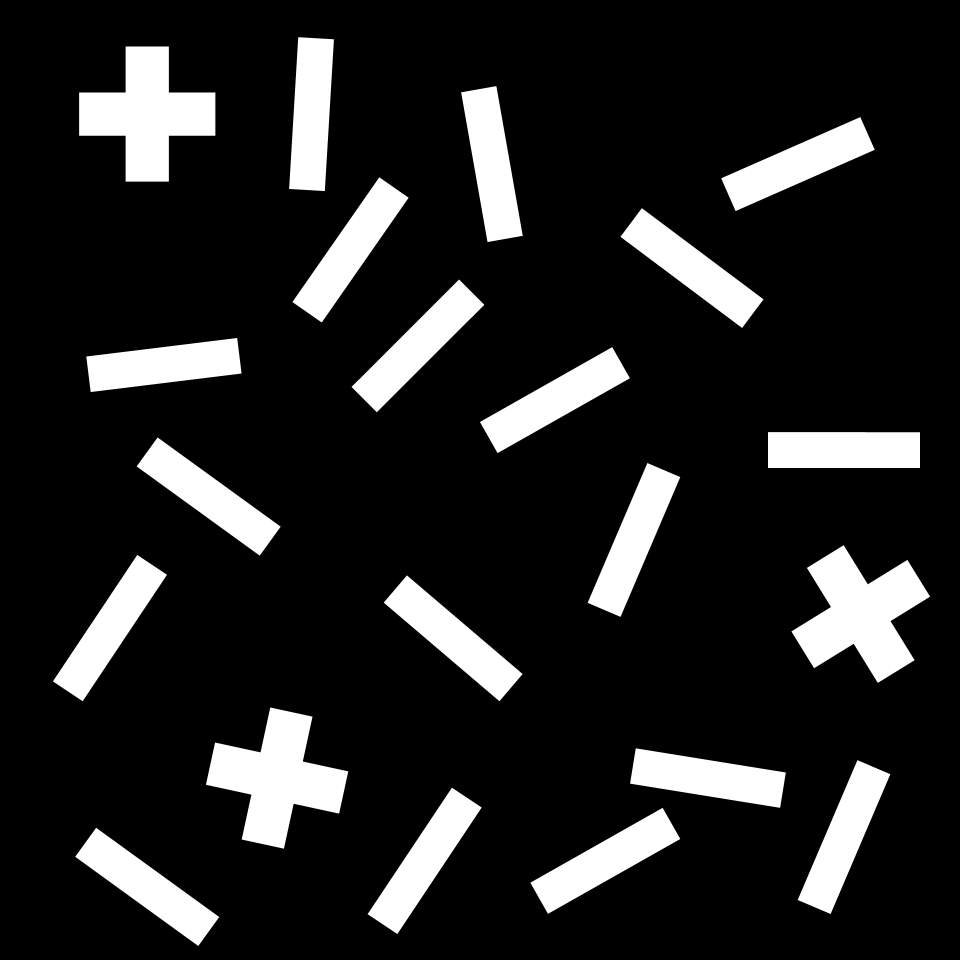}}
\hspace{1cm}
\subfigure[class 2]{\includegraphics[width=0.15\linewidth]{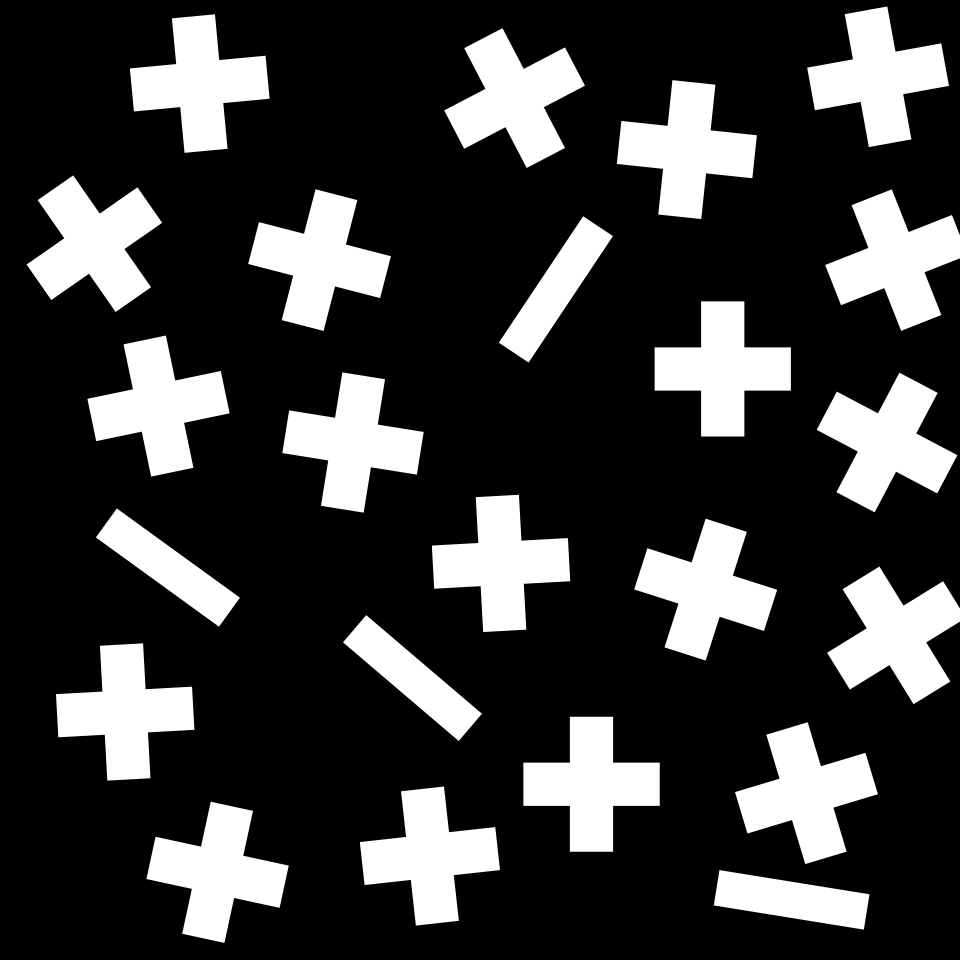}}
\caption{2D schematic illustrations of the 3D synthetic textures.}
\label{fig:synthetic}
\end{figure}

The second dataset is a subset of the American National Lung Screening Trial (NLST) that was annotated by radiologists at the University Hospitals of Geneva (HUG) \cite{MHB2019}. The dataset includes 485 pulmonary nodules from distinct patients in CT, among which 244 were labeled benign and 241 malignant.
We zero-pad or crop the input volumes (originally ranging from $16\times16\times16$ to $128\times128\times128$) to the size $64\times64\times64$. 
We use balanced training and test splits with 392 and 93 volumes respectively.
Examples of 2D slices of the lung nodules are illustrated in Fig. \ref{fig:nlst}.
The Hounsfield units of the training and test volumes are clipped in the range $[-1000,400]$, then normalized with zero mean and unit variance (using the mean and variance of cropped training volumes).

\begin{figure}[h!]
\centering
\subfigure[Benign nodule]{\includegraphics[width=0.3\linewidth]{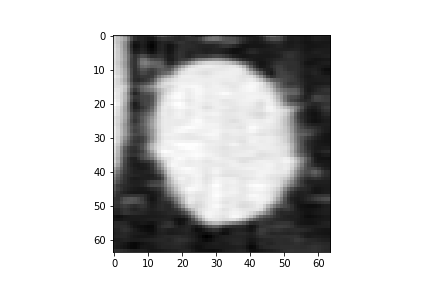}}
\subfigure[Malignant nodule]{\includegraphics[width=0.3\linewidth]{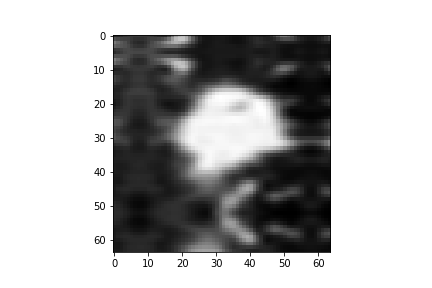}}
\caption{2D slices from 3D volumes of benign and malignant pulmonary nodules.}
\label{fig:nlst}
\end{figure}

\subsection{Network Architecture}\label{sec:architecture}
The architecture details of the CNNs are provided in this section\footnote{Code available on GitHub \url{github.com/v-andrearczyk/lri-cnn}.}. 
The first layer of the LRI networks consists of one of the LRI layers (G-LRI, S-LRI or SSE-LRI). The responses are aggregated using spatial GAP after the first layer, similarly to \cite{AnW2016} as described in Section \ref{sec:equiv_local_oper}. This pooling aggregates the LRI operator responses into a single scalar per feature map and is followed by Fully Connected (FC) layers. For the nodule classification experiment,  we average the responses inside the nodule masks instead of across the entire feature maps. 
This operation is performed to focus on the texture inside the nodule. The receptive fields are small and the sizes of the nodules vary largely across cases, making a standard GAP less appropriate for the proposed study.
For the synthetic experiment, we connect directly the final softmax FC layer with a cross-entropy loss. For the nodule classification, we use an intermediate fully connected layer with 128 neurons before the same final layer.
Standard Rectified Linear Units (ReLU) activations are employed. The networks are trained using Adam optimizer with $\beta_1=0.99$ and $\beta_2=0.9999$ and a batch size of 8.
Other task-specific parameters are: for the synthetic experiment (kernel size $7\times7\times7$, stride 1, 2 filters and 50,000 iterations), for the nodule classification experiment (kernel size $9\times9\times9$, stride 2, 4 filters and 10,000 iterations).
The number of iterations was fixed to these values as the networks reach a plateau beyond these values.

We compare the proposed architectures to a network with the same architecture but with a standard 3D convolutional layer, referred to as Z3-CNN.

\subsection{Weights Initialization}\label{sec:initialization}
The SHs are normalized to $\left\Vert Y_{n,m} \right\Vert_2 = 1 $. The coefficients are then randomly initialized by a normal distribution with $\mathrm{C}_{i,n}[m] \sim \mathcal{N}(0,\,\sigma^{2})$, with  $\sigma^2=\frac{2}{n_{in}(N+1)^2}$ and $n_{in}$
is the number of input channels (generally 1), the radial profiles are initialized to $h_{i,n}(\rho) \sim \mathcal{N}(0,\,1)$ and the biases to zero. This initialization is inspired by ~\cite{he2015delving,weiler2017learning} in order to avoid vanishing and exploding activations and gradients.
\section{Experimental Results}
\label{sec:exp}

In this section, we experimentally evaluate and compare standard CNNs (Z3-CNN with or without rotational data augmentation), the three proposed approaches to achieve LRI image analysis (i.e. G-LRI, S-LRI and SSE-LRI) as well as global RI with G-RI and S-RI.
The two datasets and tasks described in Section~\ref{sec:datasets} are used. 
The approximation capability of the SH-based parametric representation is first evaluated in Section~\ref{sec:SHapproxCapability}.
Z3 and LRI approaches are then compared in Section~\ref{sec:Z3vsLRI}. The importance of LRI when compared to global RI is investigated in Section~\ref{sec:localVSglobalRI}.
Finally, the complexity of networks is compared in terms of computational time and number of trainable parameters in Section~\ref{sec:networkComplexity}. The results will be discussed in Section~\ref{sec:discussions}.

\subsection{SH Parametric Approximation Capability}\label{sec:SHapproxCapability}
Fig. \ref{fig:M1_s} compares standard 3D kernels (Z3-CNN) with the SH parametric representation (S-LRI with $M=1$ tested orientation) using either polar separable (i.e. $h$) or non-polar separable (i.e. $h_n$) implementations. 
Based on a trade-off between complexity and performance, we select the maximum degree $N=3$ in the following experiments. 
We set $N=2$ for the lung nodule dataset with a similar analysis.
 
\begin{figure}[htbp]
\centering
\begin{tikzpicture}[]
\begin{axis}[legend pos=outer north east, xlabel={}, ylabel={acc.},width=9cm,height=4cm]
\addplot [dashed, no markers, thick, red]
      table[x=x,y=y, meta=label] {
    x       y       label
    0       78.8    b
    7       78.8    b
};
\addplot[mark=*] 
plot [error bars/.cd, y dir = both, y explicit]
      table[x=x,y=y, y error=ey, meta=label] {
    x       y       ey  label   nparam
    0       64.0    4.56   a       24
    1       65.0    1.41  a       30
    2       74.8    2.38   a       40
    3       74.8    3.03   a       54
    4       74.5    2.75   a       72
    5       79.5    1.83   a       94
    6       77.7    2.08   a       120
    7       79.4    1.71   a       150
};
\addplot[mark=*,green] 
plot [error bars/.cd, y dir = both, y explicit]
      table[x=x,y=y, y error=ey, meta=label] {
    x       y       ey  label   nparam
    0       64.0    4.56   a       24
    1       74.5    2.24   a       44
    2       79.4    2.09   a       68
    3       83.4    1.74   a       96
    4       85.5    2.00   a       128
    5       84.2    1.40   a       164
    6       85.4    1.87   a       204
    7       88.1    0.63   a       248
};
\legend{Z3 $694$ param.,S-LRI-$h$,S-LRI-$h_n$,random 50\%}
\end{axis}
\node[] at (0.55,-0.7) {{\footnotesize{}$24$/{\color{green}$24$}}};
\node[] at (1.4,-0.7) { {\footnotesize{}$30$/{\color{green}$44$}}};
\node[] at (2.3,-0.7) {  {\footnotesize\textbf{}{}$40$/{\color{green}$68$}}};
\node[] at (3.1,-0.7) {  {\footnotesize{}$54$/{\color{green}$96$}}};
\node[] at (4.0,-0.7) {  {\footnotesize{}$72$/{\color{green}$128$}}};
\node[] at (5,-0.7) {  {\footnotesize{}$94$/{\color{green}$164$}}};
\node[] at (6.0,-0.7) {  {\footnotesize{}$120$/{\color{green}$204$}}};
\node[] at (7.1,-0.7) {  {\footnotesize{}$150$/{\color{green}$248$}}};
\node[] at (8.5,-0.22) {$N$};
\node[] at (8.6,-0.67) {$\#$ param.};
\end{tikzpicture}
\caption{Comparison of standard 3D kernels (Z3) and the SH parametric representation (S-LRI) with varying maximum degree $N$ using a single orientation $M=1$ (i.e. not using the steering capacity) on the synthetic 3D texture dataset. The polar separable and non-polar separable versions are respectively denoted S-LRI-$h$ and S-LRI-$h_n$. 
The average accuracy (random 50\%) and standard error (10 repetitions) are reported as well as the numbers of parameters.}
\label{fig:M1_s}
\end{figure}
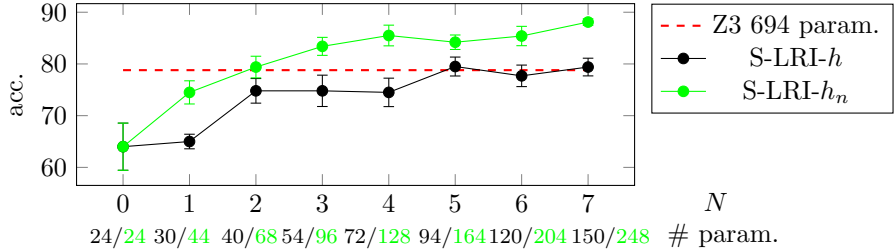

In Fig.~\ref{fig:M1_sse}, we report the performance of the SSE-LRI with varying maximal degree $N$ on the synthetic dataset. These results do not aim at evaluating the parametric representation 
but rather to evaluate the influence of $N$ for the SSE-LRI and in order to choose a value (also $N=3$ and $N=2$ for the two datasets respectively) for further comparisons in the following experiments. 

\begin{figure}
\centering
\begin{tikzpicture}[]
\begin{axis}[legend pos=outer north east, xlabel={}, ylabel={acc.},width=9cm,height=4cm]
\addplot [dashed, no markers, thick, red]
      table[x=x,y=y, meta=label] {
    x       y       label
    0       78.8    a
    6       78.8    a
};
\addplot[mark=*]
plot [error bars/.cd, y dir = both, y explicit]
      table[x=x,y=y, y error=ey, meta=label] {
    x       y       ey  label  
    0       72.9    3.8   b      
    1       84.8    1.82   b      
    2       88.5    0.71   b      
    3       90.0    0.5   b      
    4       90.6    0.25   b      
    5       91.2    0.19   b      
    6       91.7    0.34   b       
};
\addplot[mark=*, green]
plot [error bars/.cd, y dir = both, y explicit]
      table[x=x,y=y, y error=ey, meta=label] {
    x       y       ey  label  
    0       72.9    3.8   c        
    1       88.4    1.71   c      
    2       91.0    0.56   c      
    3       91.0    0.25   c      
    4       91.7    0.19   c      
    5       92.3    0.3   c      
    6       91.9    0.16   c       
};
\legend{Z3 694 param., SSE-LRI-$h$,SSE-LRI-$h_n$}
\end{axis}
\node[] at (0.55,-0.7) {{\footnotesize{}$22$/{\color{green}$22$}}};
\node[] at (1.6,-0.7) { {\footnotesize{}$28$/{\color{green}$42$}}};
\node[] at (2.6,-0.7) {  {\footnotesize\textbf{}{}$34$/{\color{green}$62$}}};
\node[] at (3.65,-0.7) {  {\footnotesize{}$40$/{\color{green}$82$}}};
\node[] at (4.7,-0.7) {  {\footnotesize{}$46$/{\color{green}$102$}}};
\node[] at (5.75,-0.7) {  {\footnotesize{}$52$/{\color{green}$122$}}};
\node[] at (6.9,-0.7) {  {\footnotesize{}$58$/{\color{green}$142$}}};
\node[] at (8.5,-0.22) {$N$};
\node[] at (8.6,-0.67) {$\#$ param.};
\end{tikzpicture}
\caption{Average accuracy (random 50\%) and standard error (10 repetitions) on the synthetic dataset for the SSE-LRI with varying values of $N$ and comparison with the standard 3D-CNN (Z3).}
\label{fig:M1_sse}
\end{figure}
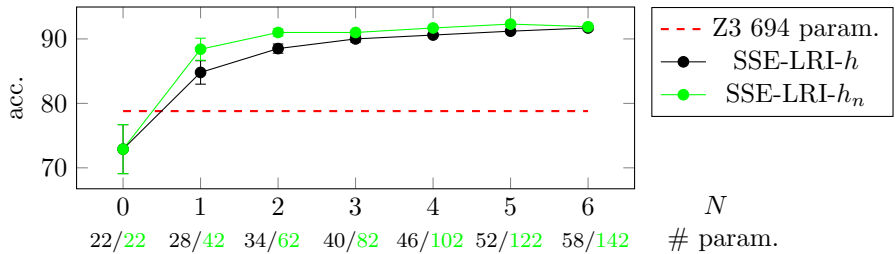

\subsection{Comparing Standard and LRI Architectures}\label{sec:Z3vsLRI}

The influence of the number of tested orientations $M$ is investigated in  Fig.~\ref{fig:results_m} using the S-LRI for both the synthetic and lung nodule datasets.
These results illustrate the benefit of a fine rotation sampling, as well as the better performance of the non-polar separable kernels. 
\begin{figure}
\begin{center}
\begin{minipage}{.49\textwidth}
\begin{tikzpicture}
\hspace{-0.3cm}
\begin{axis}[
    ybar,
    ylabel={acc.},
    enlargelimits=0.15,
    legend style={at={(0.5,-0.15)},
      anchor=north,legend columns=-1},
    symbolic x coords={$M$=1,$M$=4,$M$=24,$M$=72},
    xtick=data,
    ]
\addplot+[nodes near coords,error bars/.cd, y dir=both,y explicit] coordinates {($M$=1,74.8)+-(1.0,1.0)  ($M$=4,85.0)+-(0.5,0.5) ($M$=24,91.3)+-(0.4,0.4) ($M$=72,92.0)+-(0.2,0.2)};
\addplot+[nodes near coords,error bars/.cd, y dir=both,y explicit] coordinates {($M$=1,83.4)+-(0.5,0.5) ($M$=4,89.6)+-(0.6,0.6) ($M$=24,93.5)+-(0.1,0.1) ($M$=72,94.2)+-(0.1,0.1)};
\addplot[black,dashed,mark=none, sharp plot, line width=0.6pt]
coordinates {($M$=1,78.8) ($M$=72,78.8)};
\legend{S-LRI-$h$,S-LRI-$h_n$}
\end{axis}
\end{tikzpicture}
\end{minipage}
\begin{minipage}{.49\textwidth}
\begin{tikzpicture}
\begin{axis}[
    ybar,
    enlargelimits=0.15,
    legend style={at={(0.5,-0.15)},
      anchor=north,legend columns=-1},
    symbolic x coords={$M$=1,$M$=4,$M$=24,$M$=72},
    xtick=data,
    ]
\addplot+[nodes near coords,error bars/.cd, y dir=both,y explicit] coordinates {($M$=1,76.3)+-(0.38,0.38)  ($M$=4,79.0)+-(0.3,0.3)  ($M$=24,81.9)+-(0.33,0.33)  ($M$=72,80.7)+-(0.74,0.74)};
\addplot+[nodes near coords,error bars/.cd, y dir=both,y explicit] coordinates {($M$=1,80.5)+-(0.3,0.3)  ($M$=4,81)+-(0.3,0.3)  ($M$=24,82.8)+-(0.41,0.41)  ($M$=72,84.2)+-(0.34,0.34)};
\addplot[black,dashed,mark=none, sharp plot, line width=0.6pt]
coordinates {($M$=1,80.0) ($M$=72,80.0)};
\legend{S-LRI-$h$,S-LRI-$h_n$}
\end{axis}
\end{tikzpicture}
\end{minipage}
\end{center}
\caption{Average accuracy (\%) and standard error (10 repetitions) of the S-LRI with maximum degree $N=3$ and $N=2$ for the synthetic and NLST datasets respectively} and varying numbers of orientations $M$. Left: synthetic dataset; right: lung nodule dataset.
Black dashed lines represent the accuracy of the Z3-CNN.
\label{fig:results_m}
\end{figure}
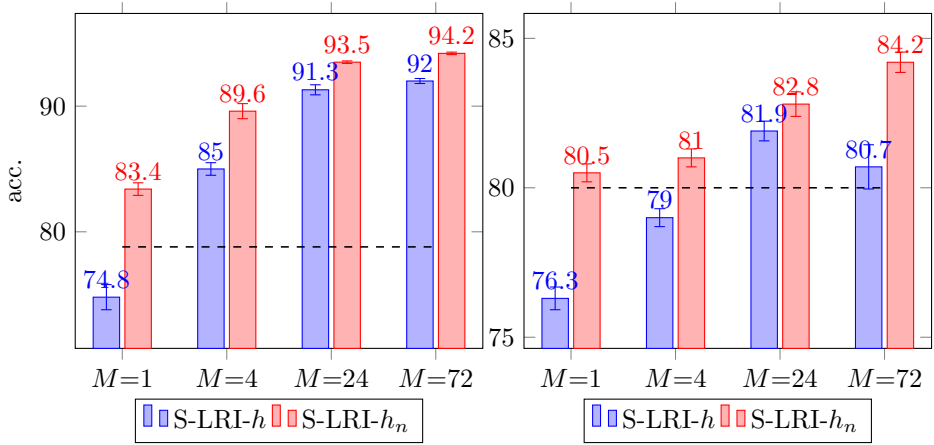
We now use the best reported values of $M$ and compare all proposed LRI approaches to the standard Z3-CNN.
The results are summarized in Tables~\ref{tab:res_all_synthetic} and~\ref{tab:res_all_nlst} for the synthetic experiment and the lung classification experiment, respectively.

\begin{table}[h]
\centering
\begin{tabular}{lccccc}
  \bfseries model & \textbf{\textit{N}} & \textbf{\textit{M}} & \bfseries \# filters & \bfseries \# param. & \bfseries accuracy\tiny{$\pm{\sigma}$}  \\
  \hline  
  Z3            & - & - & 2 & 694 & 78.8\tiny{$\pm{7.1}$} \\ 
  Z3            & - & - & 144 & 49,826 & 94.0\tiny{$\pm{0.7}$} \\ 
  
  G-LRI         & - & 24 & 2 & 694 & 89.0\tiny{$\pm{5.1}$} \\ 
  
  SSE-LRI-$h$   & 3 & - & 2 & 40 & 90.1\tiny{$\pm{1.5}$} \\ 
  
  SSE-LRI-$h_n$ & 3 & - & 2 & 82 & 91.0\tiny{$\pm{0.8}$} \\ 
  
S-LRI-$h$    & 3 & 24 & 2 & 54 & 91.3\tiny{$\pm{4.4}$} \\ 
  S-LRI-$h$     & 3  & 72 & 2 & 54 & 92.0\tiny{$\pm{1.9}$} \\ 
  
  S-LRI-$h_n$   & 3 & 24 & 2 & 96 & 93.5\tiny{$\pm{0.8}$} \\ 
  S-LRI-$h_n$   & 3 & 72 & 2 & 96 & \textbf{94.2\tiny{$\pm{1.1}$}} \\ 
  \end{tabular}
\caption{Average accuracy (\%) and standard deviation on the synthetic 3D local rotation dataset of all LRI approaches and comparison with a standard CNN (Z3).
}
\label{tab:res_all_synthetic}
\end{table}

  
  
  
  
  
  

\begin{table}[h]
\centering
\begin{tabular}{lccccc}
  \bfseries model & \textbf{\textit{N}} & \textbf{\textit{M}} & \bfseries \# filters & \bfseries \# param. & \bfseries accuracy\tiny{$\pm{\sigma}$}  \\
  \hline  
  Z3            & - & - & 4 & 3,818 & 80.0\tiny{$\pm{1.7}$} \\ 
  Z3            & - & - & 96 & 82,754 & 81.3\tiny{$\pm{2.2}$} \\ 
  
  G-LRI         & - & 24 & 4 & 3,818 & \textbf{87.7\tiny{$\pm{2.2}$}} \\ 
  
  SSE-LRI-$h$   & 2 & - & 4 & 1,966 & 81.1\tiny{$\pm{2.2}$} \\ 
  
  SSE-LRI-$h_n$ & 2 & - & 4 & 2,030 & 81.3\tiny{$\pm{2.6}$} \\ 
  
  S-LRI-$h$     & 2 & 24 & 4 & 970 & 81.9\tiny{$\pm{3.3}$} \\ 
  
  S-LRI-$h_n$   & 2 & 72 & 4 & 1,034 & 84.2\tiny{$\pm{3.4}$} \\ 
  \end{tabular}
\caption{Average accuracy (\%) and standard deviation on the pulmonary nodule classification of all LRI approaches and comparison with a standard CNN (Z3). 
}
\label{tab:res_all_nlst}
\end{table}

\subsection{Comparing Local and Global RI}\label{sec:localVSglobalRI}
The importance of LRI is investigated in this section by comparing LRI, global Rotation Invariance (RI) and rotational data augmentation. The latter consists in randomly rotating the volumes by right-angle rotations during training.
Corresponding results are reported in Tables~\ref{tab:ri_lri_synthetic} and~\ref{tab:ri_lri_nlst} for the synthetic and lung nodule datasets, respectively.

\begin{table}[h]
\centering
\begin{tabular}{lccccc}
  \bfseries model & \textbf{\textit{N}} & \textbf{\textit{M}} & \bfseries \# filters & \bfseries \# param. & \bfseries accuracy\tiny{$\pm{\sigma}$}  \\
  \hline  
  Z3            & - & - & 2 & 694 & 78.8\tiny{$\pm{7.1}$} \\ 
  Z3 augm.      & - & - & 2 & 694 & 84.0\tiny{$\pm{5.2}$} \\ 
  G-RI          & - & 24 & 2 & 694 & 79.0\tiny{$\pm{5.6}$} \\ 
  G-LRI         & - & 24 & 2 & 694 & 89.0\tiny{$\pm{5.1}$} \\ 
  S-RI-$h_n$    & 3 & 72 & 2 & 94 & 85.9\tiny{$\pm{5.5}$} \\ 
  S-LRI-$h_n$   & 3 & 72 & 2 & 96 & \textbf{94.2\tiny{$\pm{1.1}$}} \\ 
  \end{tabular}
\caption{Average accuracy (\%) and standard deviation on the synthetic 3D local rotation dataset comparing LRI, RI and data augmentation (random right-angle rotations).
}
\label{tab:ri_lri_synthetic}
\end{table}


\begin{table}[h]
\centering
\begin{tabular}{lccccc}
  \bfseries model & \textbf{\textit{N}} & \textbf{\textit{M}} & \bfseries \# filters & \bfseries \# param. & \bfseries accuracy\tiny{$\pm{\sigma}$}  \\
  \hline  
  Z3            & - & - & 4 & 3,818 & 80.0\tiny{$\pm{1.7}$} \\ 
  Z3 augm.      & - & - & 4 & 3,818 & 82.2\tiny{$\pm{3.3}$} \\ 
  G-RI          & - & 24 & 4 & 3,818 & 81.8\tiny{$\pm{2.1}$} \\ 
  G-LRI         & - & 24 & 4 & 3,818 & \textbf{87.7\tiny{$\pm{2.2}$}} \\ 
  S-RI-$h_n$    & 2 & 72 & 4 & 1,030 & 81.8\tiny{$\pm{3.1}$} \\ 
  S-LRI-$h_n$   & 2 & 72 & 4 & 1,034 & 84.2\tiny{$\pm{3.4}$} \\ 
  \end{tabular}
\caption{Average accuracy (\%) and standard deviation on the lung nodule dataset comparing LRI, RI and data augmentation (random right-angle rotations).
}
\label{tab:ri_lri_nlst}
\end{table}

\subsection{Networks Complexity: Computational Time and Trainable Parameters}\label{sec:networkComplexity}
The number of trainable parameters and the computational time are reported in Table~\ref{tab:computation_comparison}, where the Z3, S-LRI and SSE-LRI are compared. Their polar versus non-polar separable versions are also detailed.
The calculation of the number of trainable parameters is provided in~\ref{app:number_params}.

\begin{table}[h]
\centering
\resizebox{\textwidth}{!}{
\begin{tabular}{c|c|ccc|ccc|ccc|ccc}
  model & Z3 & \multicolumn{3}{c}{S-LRI-$h$} & \multicolumn{3}{|c}{S-LRI-$h_n$} & \multicolumn{3}{|c}{SSE-LRI-$h$} & \multicolumn{3}{|c}{SSE-LRI-$h_n$} \\
  N & - & 0 & 3 & 6 & 0 & 3 & 6 & 0 & 3 & 6 & 0 & 3 & 6\\ 
  \hline  
  \# param. & 694   & 24 & 54  & 120   & 24    & 96    & 204    & 22    & 40    & 58    & 22  & 82    & 142 \\ 
  time      & 28    & 56 & 95  & 196    & 56    & 97    & 211    & 30    & 64    & 146  & 30  & 64    & 146
  \end{tabular}
  }
\caption{Computational and parameters comparison on the synthetic dataset with the setups of Table~\ref{tab:res_all_synthetic} (some extra setups are included for comparison). The computational time is measured in seconds for 1,000 iterations trained on a Tesla K80 GPU.
  }
\label{tab:computation_comparison}
\end{table}


\section{Discussions}
\label{sec:discussions}
We discuss and interpret the results detailed in Section~\ref{sec:exp} in terms of the general importance of LRI image analysis (Section~\ref{sec:importanceLRI}), optimal LRI design (Section~\ref{sec:optimalLRI}), as well as kernel compression (reduction of trainable parameters) and interpretability (Section~\ref{sec:compression}).
\subsection{Importance of LRI}\label{sec:importanceLRI}
The results reported in the previous section demonstrated the importance of LRI image analysis in the proposed experiments. 
Particularly, best results are obtained with LRI architectures (S-LRI-$h_n$ on the synthetic dataset and G-LRI for pulmonary nodule classification) as reported in Tables~\ref{tab:res_all_synthetic} and~\ref{tab:res_all_nlst}. 
Besides, LRI performs significantly better than RI in these experiments where local patterns occur at random orientations, as shown in Tables~\ref{tab:ri_lri_synthetic} and~\ref{tab:ri_lri_nlst}. 
In addition, despite improving the performance of the standard Z3-CNN, rotation data augmentation is not sufficient to obtain LRI (Table~\ref{tab:ri_lri_synthetic} and~\ref{tab:ri_lri_nlst}).
Adding more filters\footnote{The number of filters is multiplied by the number of orientations used in the LRI methods for a fair comparison.} to the Z3-CNN (second rows of Tables~\ref{tab:res_all_synthetic} and~\ref{tab:res_all_nlst}) allows learning filters at different orientations at the heavy cost of a large number of parameters and convolution operations without reaching the performance of LRI networks.

In the SSE-LRI, the number of output feature maps and trainable parameters of the LRI convolution increases with $N$ (see \eqref{eq:Gsse}). To show that the performance gain is not solely due to more feature maps and parameters but rather to a better approximation capability, we evaluate the SSE-CNN with $N=0$ and more output channels ($C=8$ instead of $C=2$). This setup relies on a number of output feature maps that is equal to the SSE-CNN with $N=3$ in Table~\ref{tab:res_all_synthetic}.
Yet, the accuracy of the former is $81.5\%\scriptscriptstyle{\pm{3.2}}$ versus $90.1\%\scriptscriptstyle{\pm{1.5}}$ for the latter. This result highlights the relevance of SSE quantities extracted at various degrees $n$ and, in turn, the importance of directional sensitivity. 
Note that the synthetic patterns do not have exactly the same zero frequency, explaining the fact that both S-LRI and SSE-LRI designs based on $N=0$, i.e. using directionally insensitive filters, can discriminate the two classes to some extent.

\subsection{Comparison of LRI Methods}\label{sec:optimalLRI}
LRI can be obtained by G-convolution, steering, or invariants computed from the SH responses, as summarized in Fig.~\ref{fig:lri_overview}.
The results of the parametric representation (Fig.~\ref{fig:M1_s}), confirmed by the following results using steerability or spherical energy (Figures~\ref{fig:M1_sse},~\ref{fig:results_m} and Tables~\ref{tab:res_all_synthetic},~\ref{tab:res_all_nlst}), show that non-polar separable (radial profile $h_n$) filters perform better than polar separable ones (radial profile $h$). 
In particular, it provides more flexibility to learn the optimal combination of different degrees in \eqref{eq:formoffilterinterm} with individual radial profiles, resulting in a better parametric approximation capability with a limited increase of parameters. To ensure that the performance gain is not only due to the increase of parameters, we incrementally increased from two to five the number of filters of the polar separable design for the synthetic dataset (number of trainable parameters ranging from 54 to 132). The highest accuracy with $M=72$ orientations is 92.4\%, outperformed by the non-polar separable design with 94.2\%.

A fine sampling of orientations ($M$=72) is beneficial to the S-LRI (see Fig.~\ref{fig:results_m}), particularly outperforming the other architectures on the synthetic dataset (Table~\ref{tab:res_all_synthetic}). The polar-separable S-LRI-$h$, however, may be too simplistic to benefit from a finer orientation sampling on the nodule dataset (Fig.~\ref{fig:results_m}, right).
The SSE-LRI offers a trade-off between performance and computation. It does not require steering the SH responses (the locally rotation equivariant representation not being required), resulting in a reduction of memory and operations requirements but at the cost of a lower kernel specificity: the solid spherical energy mixes responses of different SH patterns (for $n>0$) as well as inter-degree phases, thus discarding some potentially valuable discriminatory information for later layers.

On the NLST dataset, the G-LRI performs better than the S-LRI. The local patterns in the lung nodule dataset may be easier to represent in a non-parametric form as compared to the synthetic patterns.

\subsection{Compression and Interpretability}\label{sec:compression}
The number of trainable parameters is largely reduced when using S-LRI and SSE-LRI as compared to a standard Z3-CNN. 
We identify two distinct factors enabling parameter reduction, namely the \textit{weight sharing} across rotations and the \textit{parametric representation}. 
The \textit{weight sharing} is present, among others, in the G-LRI where the same kernels account for multiple orientations. Similarly, the S-LRI and SSE-LRI architectures share trainable radial profiles and harmonic coefficients to obtain responses to rotated kernels.
The \textit{parametric representation}, on the other hand, is used in~\cite{andrearczyk2019exploring,weiler2017learning,WGT2016} by learning a combination of basis filters instead of every voxel of the kernels. A steerable basis was also used in \cite{MUD2020} for a fast rotational sparse coding. 
Fig.~\ref{fig:M1_s} shows the performance of the parametric representation (with a single orientation of the S-LRI), even outperforming the standard Z3-CNN that contains many more parameters.  
The proposed SSE-LRI approach goes one step further as there is no need to learn an explicit full parametric representation of the kernels. For computing the SSE-LRI invariants, we only calculate the norm of the SHs responses, resulting in a large reduction of the number of trainable parameters.

One bottleneck with both G-CNN~(\cite{winkels2019pulmonary}) and S-LRI~(\cite{andrearczyk2019exploring}) is the GPU memory usage to store $M$ 3D response maps (i.e. all orientations of the equivariant representation) before orientation pooling.
This memory consumption is drastically reduced in the SSE-LRI by computing invariants on the SHs responses rather than calculating responses at all orientations. Note that the S-LRI and SSE-LRI implementations only use existing TensorFlow functions and the computation time could be further improved by efficient parallelization and CUDA programming.

Finally, in terms of network interpretability, the hard-coded equivariance and invariance enforces a geometric structure of the hidden features improving the transparency and decomposability of the network where transformations (translations and rotations) in the inputs result in predictable transformations in the activations (see~\cite{lipton2018mythos,WGT2016, worrall2018cubenet,cheng2018rotdcf} for discussions on the matter).

\section{Conclusion}
This paper explored the use of LRI in the context of 3D texture analysis. Three architectures were proposed and compared with standard 3D-CNN, global RI and data augmentation.
The results showed the importance of LRI in medical imaging and, more generally, in texture analysis where repeated patterns occur at various locations and orientations. In particular, we showed that data augmentation, commonly used in CNN training, is not sufficient to learn such an invariance and specific architectures with built-in invariance are beneficial.

In future work, we plan to explore deeper networks, building deep architectures on top of LRI layers.

\section{Acknowledgments}
This work partially supported by the Swiss National Science Foundation (grant 205320\_179069), the Swiss Personalized Health Network (IMAGINE and QA4IQI projects), as well as a hardware grant from NVIDIA.
\section{Declarations of interest}
Declarations of interest: none.
\section*{References}
\bibliography{main}

\appendix

\section{Proof of Proposition \ref{prop:linearRI}}
\label{app:Glinear}

First of all, any convolution operator $\mathcal{G}$ is equivariant to translations by construction. We therefore focus on the equivariance to rotations. 

Let $f$ and $g$ be two functions and $\mathrm{R}_0 \in SO(3)$. A simple change of variable implies that
$   \left(  f * g(\mathrm{R}_0^{-1} \cdot ) \right)
   = \left( f ( \mathrm{R}_0 \cdot ) * g\right) ( \mathrm{R}_0^{-1} \cdot). $
   Exploiting this relation, we deduce that, for any image $I \in L_2(\mathbb{R}^3)$ and any rotation $\mathrm{R}_0 \in SO(3)$, 
   \begin{equation} \label{eq:truc}
        \left(  h (\mathrm{R}_0^{-1} \cdot ) * I \right) 
        = 
        \left(  h  * I (\mathrm{R}_0 \cdot )  \right) ( \mathrm{R}_0^{-1} )
        = \mathcal{G} \{  I (\mathrm{R}_0 \cdot ) \}( \mathrm{R}_0^{-1} ). 
   \end{equation}
Assume that $\mathcal{G}$ is equivariant to rotations, then $\mathcal{G} \{  I (\mathrm{R}_0 \cdot ) \}( \mathrm{R}_0^{-1} ) = \mathcal{G}\{ I \} = h*I$. Therefore,  $\left(  h (\mathrm{R}_0^{-1} \cdot ) * I \right)  = h*I$  for any $I$, which implies the equality $ h (\mathrm{R}_0^{-1} \cdot ) = h$ for any rotation and $h$ is isotropic.
Reciprocally, if $h$ is isotropic, \eqref{eq:truc} shows that $\mathcal{G} \{  I (\mathrm{R}_0 \cdot ) \}( \mathrm{R}_0^{-1} ) = \mathcal{G}\{ I \}$ for any $I$ and any rotation $\mathrm{R}_0$, which is equivalent to the rotation equivariance of $\mathcal{G}$.

Moreover, the values of $\mathcal{G}\{ I \}(\bm{x})$ depend on local image values $I(\bm{y})$ for $\bm{y}-\bm{x}$ in the support of $h$. Hence, $\mathcal{G}$ is local if and only if the support of $h$ is compact.

Finally, $\mathcal{G}$ is LRI if and only if $h$ is isotropic (for the global equivariance to rotations) and compactly supported (for the locality).

\section{Spherical Harmonics} \label{app:SH}

The family of SHs is denoted by $(Y_{n,m})_{n \geq 0, m \in \{-n , \ldots , n\}}$, where $n$ is called the degree and $m$ the order of $Y_{n,m}$. SHs form an orthonormal basis for square-integrable functions in the $2D$-sphere $\mathbb{S}^2$. They are defined as~(\cite{DH1994})
\begin{equation}
	Y_{n,m}(\theta,\phi) = A_{n,m} P_{n,\lvert m \rvert} (\cos( \theta) ) \mathrm{e}^{\mathrm{j} m \phi},
\end{equation}
with $A_{n,m}=(-1)^{(m + \lvert m \rvert )/2}
\left(\frac{2n+1}{4\pi} \frac{(n-\lvert m \rvert)!}{(n+ \lvert m \rvert)!} \right)^{1/2}$ a normalization constant 
and $P_{n,\lvert m \rvert}$ the associated Legendre polynomial given for $0\leq m \leq n$ by 
\begin{equation}
P_{n,m}(x) := \frac{(-1)^m}{2^n n!} (1-x^2)^{m/2} \frac{\mathrm{d}^{n+m}}{\mathrm{d}x^{n+m}} (x^2-1)^{n}.
\end{equation}
We refer to \cite{AS1964} for more details.

\section{Real Steerable Filters}\label{sec:proof1}

\begin{proposition}
A function $f$ of the form \eqref{eq:formoffilterinterm} is real if and only if we have, for every $0 \leq n \leq N$, $-n\leq m \leq n$, 
\begin{equation} \label{eq:conditionrealhnm}
  \forall \rho \geq 0, \quad  C_n[- m] h_{n} (\rho) = (-1)^m \overline{C_n[m]} \overline{h_{n}(\rho)} .
\end{equation}
If we impose moreover that the radial profiles $h_n$ are real, then this is equivalent to 
\begin{equation} \label{eq:conditionCnmreal}
\mathrm{C}_n[-m] = (-1)^m \overline{\mathrm{C}_n[m]}
\end{equation}
for every $0 \leq n \leq N$, $-n\leq m \leq n$.
\end{proposition}


\begin{proof}
A filter $f$ is real if and only if 
$\overline{f(\rho,\theta,\phi)} = f (\rho,\theta,\phi)$ for every $(\rho,\theta,\phi)$. For filters given by \eqref{eq:formoffilterinterm}, this means that
\begin{equation} 
\sum_{n,m} \overline{C_n[m] h_{n }(\rho)  Y_{n,m}(\theta,\phi)} = \sum_{n,m} {C_n[m] h_{n }(\rho) Y_{n,m}(\theta,\phi)}.
\end{equation}
We use the symmetry of the spherical harmonics, $\overline{Y_{n,m}}=(-1)^m Y_{n,-m}$, on the left-hand side and change the sign of $m$ on the right-hand side to get
\begin{equation}
\sum_{n,m} \overline{C_n[m] h_{n }(\rho)} (-1)^m Y_{n,-m} (\theta,\phi) =\sum_{n,m}  C_n[-m] h_{n }(\rho) Y_{n,-m} (\theta,\phi).
\end{equation}
The $Y_{n,m}$ being linearly independent, we 
deduce that the filter is real if and only if \eqref{eq:conditionrealhnm} holds.
By imposing that the $h_n$ are real, i.e., $\overline{h_n} = h_n$, we obtain the expected criterion on the $\mathrm{C}_n[m]$ coefficients, which is \eqref{eq:conditionCnmreal}.
\end{proof}

\section{Equivariant Image Operators via Orientation Channels} \label{app:equivariance_s}

%
%
This result is reported for completeness, yet already proven in our previous publication~(\cite{andrearczyk2019exploring}).

\begin{proposition}
An image operator of the form 
\eqref{eq:Gg} and \eqref{eq:Gs}
is equivariant to translations and rotations in the sense of \eqref{eq:transinv} and  \eqref{eq:rotinv} and therefore LRI when $f$ is compactly supported.
\end{proposition}

\begin{proof}
The equivariance to translations uses $(I(\cdot - \bm{x}_0) * g) (\bm{x}) = (I*g) (\bm{x}-\bm{x}_0)$. Applying this to $g=f(\mathrm{R}\cdot)$, we deduce
\begin{equation}
    \mathcal{G}\{I(\cdot - \bm{x}_0)\}(\bm{x}) 
    = \max_{\mathrm{R} \in SO(3)} \lvert (I* f(\mathrm{R}\cdot))(\bm{x}- \bm{x}_0 )\rvert = \mathcal{G}\{I\}(\bm{x}-\bm{x}_0), 
\end{equation}
as expected. For rotations, we use $(I(\mathrm{R}_0\cdot) * g) (\bm{x}) = (I * g(\mathrm{R}_0^{-1}\cdot))(\mathrm{R}_0\bm{x})$ applied to $g=f(\mathrm{R}\cdot)$ to deduce
\begin{align}
      \mathcal{G}\{I(\mathrm{R}_0 \cdot )\}(\bm{x}) 
      &=
      \max_{\mathrm{R} \in SO(3)} 
      \lvert (I* f(\mathrm{R} \mathrm{R}_0^{-1}\cdot))(\mathrm{R}_0\bm{x})\rvert \\
      &=
      \max_{\mathrm{R} \in SO(3)} 
      \lvert (I* f(\mathrm{R} \cdot))(\mathrm{R}_0\bm{x})\rvert \\
      &=
            \mathcal{G}\{I \}(\mathrm{R}_0\bm{x}),
\end{align}
where the second equality simply exploits that $\mathrm{R}\mathrm{R}_0^{-1}$ describes the complete space $SO(3)$ of 3D rotations when $\mathrm{R}$ varies. The property of being LRI is a direct consequence of the equivariance by translations and rotations, together with the fact that $f$ is compactly supported.
\end{proof}

We remark that the equivariance to translations is simply due to the use of the convolution, while the equivariance to rotations requires pooling over 3D rotations in \eqref{eq:Gg} and \eqref{eq:Gs}.

\section{Equivariant Image Operators via \mbox{SH invariants}} \label{app:equivariance_sse}
%
%
\begin{proposition}
An image operator of the form \eqref{eq:Gsse}
is equivariant to translations and rotations in the sense of \eqref{eq:transinv} and \eqref{eq:rotinv}, and therefore LRI when the $h_n$ are compactly supported. 
\end{proposition}

\begin{proof}
Given a filter $g$ (e.g. $g=h_nY_{n,m}$),  from the relation $(I(\cdot - \bm{x}_0)* g) (\bm{x})= (I*g) (\bm{x} - \bm{x}_0)$, we deduce that
$\mathcal{G}_n^{SSE} \{ I (\cdot - \bm{x}_0) \} (\bm{x})
    = 
    \sum_{m=-n}^n | ( I * h_nY_{n,m} ) (\bm{x} - \bm{x}_0)|^2 = \mathcal{G}_n^{SSE}\{I\}(\bm{x}-\bm{x}_0)$,
which is \eqref{eq:transinv}. 
Now, using that $(I(\mathrm{R}_0 \cdot) * g) (\bm{x}) = (I* g(\mathrm{R}_0^{-1} \cdot ) ) ( \mathrm{R}_0 \bm{x})$ and~\eqref{eq:YnmYnm'}, we have
\begin{align} \label{eq:computeGnrotate}
    \mathcal{G}_n^{SSE}\{I ( \mathrm{R}_0 \cdot )\}(\bm{x})
    &=
    \sum_{m=-n}^n | ( I * (h_n Y_{n,m} (\mathrm{R}_0^{-1} \cdot)) ) (\mathrm{R}_0 \bm{x})|^2 \nonumber \\
    &=
    \sum_{m=-n}^n \left\lvert \sum_{m'=-n}^n \mathrm{D}_{n,\mathrm{R}_0^{-1}} [m,m'] ( I* h_n Y_{n,m'}) (\mathrm{R}_0 \bm{x}) \right\rvert^2.
\end{align}
The Wigner $\mathrm{D}$-matrix being norm-preserving, we have that $$\sum_m |c_m|^2 = \sum_m \left | \sum_{m'} \mathrm{D}_{n,\mathrm{R}_0^{-1}} [m,m'] c_{m'} \right |^2$$ 
for any $\bm{c}= (c_{-n},\ldots , c_n)$. Applying this relation to $c_m = (I* h_n Y_{n,m}) (\mathrm{R}_0 \bm{x})$, we deduce \eqref{eq:rotinv} from \eqref{eq:computeGnrotate}.
Finally, the LRI is a consequence of the equivariance to global rotations and translations.
\end{proof}

\section{Number of Trainable Parameters}\label{app:number_params}
The trainable parameters include convolutional and fully connected parameters, biases and $\mathrm{C}_n[m]$ harmonic coefficients. In this appendix, we develop the calculation of their number for the S-LRI and SSE-LRI architectures.
\paragraph{S-LRI}
The number of parameters $n_{total(S-h)}$ of the polar separable S-LRI-$h$ architectures is computed as 
\begin{equation}
    n_{total(S-h)} = n_f n_r + n_f + (N+1)^2 n_f + n_f n_c+n_c , 
\end{equation} 
where $n_f$, $n_r$ and $n_c$ are the number of filters, of radial profile parameters (Section~\ref{sec:radial_profiles}) and of classes respectively. 
For example, in the synthetic experiment for $N=3$, it sums up to $2\times7+2+(3+1)^2\times2+2\times2+2=54$. 

For the non-polar separable S-LRI-$h_n$, we have a different trainable radial profile for each degree $n$, resulting in the following:
\begin{equation}
    n_{total(S-h_n)} = (N+1) n_f n_r + n_f + (N+1)^2 n_f + n_f n_c+n_c = 96 . 
\end{equation} 

\paragraph{SSE-LRI}
The number of parameters of the polar separable SSE-LRI-$h$ architectures is
\begin{equation}
    n_{total(SSE-h)} = n_f n_r + n_f  (N+1) + (N+1) n_f n_c + n_c.
\end{equation} 
In the synthetic experiment for $N=3$, it sums up to $2\times7+2\times4 +4\times2\times2+2=40$. 

For the non-polar separable SSE-LRI-$h_n$, it is
\begin{equation}
    n_{total(SSE-h_n)} = (N+1) n_f n_r + n_f  (N+1) + (N+1) n_f n_c + n_c 
 = 82.
\end{equation}



\end{document}